\documentclass[conference,12pt,onecolumn]{ieeeconf}  

\usepackage{amsfonts}
\usepackage{amsmath}
\usepackage{algorithm}
\usepackage{algorithmic}
\usepackage{graphicx}

\IEEEoverridecommandlockouts                              

\title{\LARGE \bf
A Policy Gradient Approach for Finite Horizon Constrained Markov Decision Processes
}

\author{Soumyajit Guin$^{1}$ and Shalabh Bhatnagar$^{1}$
\thanks{*This work was supported by a J.~C.~Bose Fellowship, Project No.~DFTM/02/3125/M/04/AIR-04 from DRDO under DIA-RCOE, a project from DST-ICPS, and the RBCCPS, IISc.}
\thanks{$^{1}$ The authors are with the Department of Computer Science and Automation, Indian Institute of Science, Bangalore 560012, India
        {\tt\small gsoumyajit@iisc.ac.in; shalabh@iisc.ac.in.}}%
}
\newtheorem{assumption}{Assumption}
\newtheorem{remark}{Remark}
\newtheorem{theorem}{Theorem}

\begin{document}

\maketitle
\thispagestyle{empty}
\pagestyle{empty}

\begin{abstract}

The infinite horizon setting is widely adopted for problems of reinforcement learning (RL). These invariably result in stationary policies that are optimal. In many situations, finite horizon control problems are of interest and for such problems, the optimal policies are time-varying in general. Another setting that has become popular in recent times is of Constrained Reinforcement Learning, where the agent maximizes its rewards while it also aims to satisfy some given constraint criteria. However, this setting has only been studied in the context of infinite horizon MDPs where stationary policies are optimal. We present an algorithm for constrained RL in the Finite Horizon Setting where the horizon terminates after a fixed (finite) time. We use function approximation in our algorithm which is essential when the state and action spaces are large or continuous and use the policy gradient method to find the optimal policy. The optimal policy that we obtain depends on the stage and so is non-stationary in general. To the best of our knowledge, our paper presents the first policy gradient algorithm for the finite horizon setting with constraints. We show the convergence of our algorithm to a constrained optimal policy. We also compare and analyze the performance of our algorithm through experiments and show that our algorithm performs better than some other well known algorithms.
\end{abstract}

\section{Introduction}

The Constrained Markov Decision Process (C-MDP) setting has recently received significant attention in the reinforcement learning (RL) literature due to its natural application in safe RL problems \cite{achiam2017,garcia2015}. A textbook treatment of C-MDP can be found in \cite{altman1999}. In the C-MDP framework, in addition to the long-term objective specified via single-stage rewards (that are associated with state transitions), there are also long-term constraint functions specified via additional single-stage rewards or costs. The goal then is to find an optimal policy that maximizes a long-term reward objective while satisfying prescribed constraints. RL algorithms for infinite horizon discounted C-MDP have been studied in \cite{geibel2006}. For the long-run average cost C-MDP, \cite{borkar2005} has developed the first actor-critic RL algorithm in the full state setting. RL Algorithms with function approximation have also  been developed for infinite horizon discounted cost C-MDP \cite{bhatnagar2010} as well as the long-run average cost C-MDP \cite{bhatnagar2012}.

In this paper, we present an RL algorithm for C-MDP in the finite horizon setting. Finite Horizon problems \cite{bertsekas2005,bertsekas1996} deal with situations where the agent needs to choose a finite number of actions depending on the states of the environment in order to maximize the expected sum of single-stage and the terminal reward. An optimal policy in this setting would in general be non-stationary as the choice of an action at an instant would depend not just on the state at that instant but also on the number of actions remaining to be chosen from then on so as to maximize a long-term objective. RL techniques for Finite Horizon (regular) MDP in the full state case have been discussed in \cite{garcia1998}. They give two algorithms for the tabular case: $Q_\mathcal{H}$-Learning and $R_\mathcal{H}$-Learning, and do a learning-rate analysis between them. Actor-critic type algorithms for Finite-Horizon (regular) MDPs are discussed in \cite{bhatnagar2008}. They give  four algorithms for Finite Horizon, among which three are for tabular setting, and one algorithm uses function approximation. Their algorithms for tabular setting are not scalable to large state spaces. The algorithm which uses function approximation uses zeroth order gradient instead of first order gradient, like us and also doesn't consider inequality constraints, which we do. Our RL algorithm is devised for finite-horizon C-MDP, uses function approximation, and involves actor-critic type updates. Temporal difference learning algorithms for a finite horizon setting have also recently been studied in \cite{deasis2020}. They show convergence for Q-Learning with Linear function-approximation and some general function approximations for Finite Horizon. With their motivation to make infinite horizon Q-Learning stable, they do not consider time varying transition probability and reward functions, which we do, which is a more complicated setting. 

We prove the convergence of our algorithm under standard assumptions. Our convergence guarantees gives one the power to mimic our algorithm and make algorithms using neural networks for solving complex tasks. One such task could be a portfolio management system \cite{jiang2017deep}, where a person wants to invest in the stock market for a finite amount of time. The person here needs to decide the ratio of the money they will invest in different stocks. The stock values naturally change with time, and the decisions made are time critical in nature. We also show in this paper, empirical results on a two-dimensional grid world problem using our algorithms where we observe that our algorithms clearly meet the constraint cost performance while giving a good reward performance. The other algorithms in the literature do not meet the constraint objective.  

\subsection*{Our contributions: } 
\begin{enumerate}
    \item 
We present and analyze both theoretically and experimentally the first policy gradient reinforcement learning algorithm with function approximation for Finite Horizon Constrained Markov Decision Processes. 
\item This setting differs significantly from infinite horizon problems since in the latter stationary policies are optimal unlike the finite horizon setting where invariably non-stationary policies are optimal. This is because knowledge of the time remaining for termination of the horizon often has a profound bearing on action selection in the finite horizon case.
\item We prove that our proposed algorithm converges almost surely to a constrained optimum over a tuple of parameters, one corresponding to each instant in the horizon.
\item We show a comparison of the empirical performance of our algorithm in relation to well known algorithms in the literature originally designed for infinite horizon problems. 
\item Our key observation here is that our algorithm gives a good reward performance while strictly meeting the constraint criterion at every time instant unlike the other algorithms that do meet the constraint criterion and are therefore unsuitable for Constrained Finite Horizon problems.
\end{enumerate}

\section{The framework and problem formulation}
\label{framework}

By a Markov Decision Process (MDP), we mean a stochastic process \(\{X_h\}\) taking values in a set (called the state space) $S$, and governed by a control sequence \(\{Z_h\}\) taking values in a set $A$ (called the action space), that satisfies the following controlled Markov property: \(P(X_{h+1}=s'|X_m,Z_m,m \leq h)=p_h(X_h,Z_h,s')\mbox{ }\) a.s. Also, associated with any state transition is a single-stage reward. We assume here that $S$ and $A$ are finite sets. Actions are chosen by the decision maker at instants \(h=0,1,\dots,H-1\), and the process terminates at instant \(H\), where $H$ is a given positive integer. For simplicity, we consider all actions in $A$ to be feasible in every state in $S$. A non-stationary randomized policy (NRP) is a set of \(H\) distributions defined by \(\pi:=\{\mu_0,\mu_1,...,\mu_{H-1}\}\), where \(\mu_h(s,\cdot)\) is a distribution over \(A,s \in S, h=0,1,\dots,H-1\). Here, the action \(Z_h \sim \mu_h(X_h,\cdot), \forall h\). It is easy to see under a NRP, \(\{X_h\}\) is a non-homogeneous Markov chain.

Let \(r_h(s,a,s')\) (resp.~\(g_h^{(1)}(s,a,s'),\dots,g_h^{(M)}(s,a,s')\)) be the single-stage reward (resp.~the set of single-stage costs) at instant \(h=0,1,\dots,H-1\) when the state is \(s \in S\), the action chosen is \(a \in A\) and the next state is \(s' \in S\). Let \(r_H(s)\) (resp. \(g^{(1)}_H(s),\ldots, g^{(M)}_H(s)\)) likewise denote the terminal reward (resp.~set of terminal costs) at instant \(H\) when the terminal state is \(s \in S\). Our aim is to find a NRP \(\pi^*\) that maximizes the following over all NRP \(\pi\):

\begin{equation}\label{e1}
J(\pi)=\mathbb{E}\left[r^{(k)}_H(X_H)+\sum_{h=0}^{H-1}r^{(k)}_h(X_h,Z_h,X_{h+1})\right],
\end{equation}

subject to the constraints

\begin{equation}\label{e2}
\begin{split}
S^{(k)}(\pi)=\mathbb{E}\big[g^{(k)}_H(X_H)+\sum_{h=0}^{H-1}g^{(k)}_h(X_h,Z_h,X_{h+1})\big]\leq \alpha^{(k)},
\end{split}
\end{equation}
\(k=1,\dots,M\). Here \(\alpha^{(1)},\dots,\alpha^{(M)}\) are certain prescribed threshold values. We assume there exists at least one NRP \(\pi\) for which all the inequality constraints in (\ref{e2}) are satisfied. Let \(\lambda=(\lambda^{(1)},\dots,\lambda^{(M)})^T\) denote the vector of Lagrange multipliers \(\lambda^{(1)},\dots,\lambda^{(M)} \in \mathbb{R}^{+} \cup \{0\}\) and let \(L(\pi,\lambda)\) denote the Lagrangian:

\begin{equation}
\begin{split}
L(\pi,\lambda)&=J(\pi)+\sum_{k=1}^M \lambda^{(k)} (S^{(k)}(\pi)-\alpha^{(k)})=\mathbb{E}\left[c^{\lambda}_H(X_H)+\sum_{h=0}^{H-1}c^{\lambda}_h(X_h,Z_h,X_{h+1})\right],
\end{split}
\end{equation}
where \(c^{\lambda}_H(s):=r_H(s)+\sum_{k=1}^M \lambda^{(k)} (g^{(k)}_H(s)-\alpha^{(k)})\) and \(c^{\lambda}_h(s,a,s'):=r_h(s,a,s')+\sum_{k=1}^M \lambda^{(k)} g^{(k)}_h(s,a,s')\) respectively. Let \(V^{\pi,\lambda}_h(s),s \in S\) be the value function at instant \(h=0,1,\dots,H-1\) for the relaxed MDP problem.
\begin{equation}
\begin{split}
V^{\pi,\lambda}_h(s)=\mathbb{E}\big[c^{\lambda}_H(X_H)+\sum_{t=h}^{H-1}c^{\lambda}_t(X_t,Z_t,X_{t+1})|X_h=s\big]
\end{split}
\end{equation} 
and similarly,
\begin{equation}
V^{\pi,\lambda}_H(s)=c^{\lambda}_H(s).
\end{equation}

Note that the foregoing can be written in terms of the Q-value function for \(s \in S,a \in A\) as
\begin{equation}
V^{\pi,\lambda}_h(s)=\sum_{a \in A} \mu_h(s,a) Q^{\pi,\lambda}_h(s,a)
\end{equation}
where,
\[
Q^{\pi,\lambda}_h(s,a)=\sum_{s' \in S} p_h(s,a,s')\left[c^{\lambda}_h(s,a,s')+V^{\pi,\lambda}_{h+1}(s')\right].
\]
Similarly we can define value function for the constraints. The value function for \(s \in S\) of the $k$th constraint, \(k=1,\dots,M,\) is defined as:
\begin{equation}
\begin{split}
W^{\pi,(k)}_h(s)=\sum_{a \in A} \mu_h(s,a) \sum_{s' \in S} p_h(s,a,s') \big[g_h(s,a,s')+W^{\pi,(k)}_{h+1}(s')\big]
\end{split}
\end{equation}
and,
\begin{equation}
W^{\pi,(k)}_H(s)=g^{(k)}_H(s)-\alpha^{(k)}.
\end{equation}
Let \(\beta(s_0),s_0 \in S\) be the initial state distribution. For an NRP \(\pi\), let \(Pr(s_0 \rightarrow s,h,\pi)\) be the probability of transitioning from initial state \(s_0\) to state $s$ in $h$ steps under policy \(\pi\). Also, under \(\pi\), let \(d^{\pi}_h(s)\) be the probability of reaching state \(s\) in \(h\) steps, \(h=0,1,\dots,H\).
\begin{equation}
\label{dpi}
\begin{split}
&d^{\pi}_h(s):=\sum_{s_0 \in S} \beta(s_0) Pr(s_0 \rightarrow s,h,\pi)=\sum_{s_0 \in S} \beta(s_0) \sum_{a_0 \in A}\mu_0(s_0,a_0) \sum_{s_1 \in S} p_0(s_0,a_0,s_1)\\
&\sum_{a_1 \in A}\mu_1(s_1,a_1)\sum_{s_2 \in S} p_1(s_1,a_1,s_2) \dots\sum_{a_{h-1} \in A}\mu_{h-1}(s_{h-1},a_{h-1}) p_{h-1}(s_{h-1},a_{h-1},s).
\end{split}
\end{equation}
We consider here a parameterized class of NRP \(\pi\) where \(\mu_h\) depends on parameter \(\theta_h \in \mathbb{R}^{y_h}\), \(h=0,1,\dots,H-1\). Let \(\theta\stackrel{\triangle}{=} (\theta_0,\theta_1,\ldots,\theta_{H-1})^T\). Let \(\pi_{\theta}=\{\mu_{\theta_h}(s,a), s \in S, a \in A, \theta_h \in \mathbb{R}^{y_h}, h=0,\dots,H-1\}\) denote the parameterized class of NRP. Our goal then is to find an optimal \(\theta^*\) that maximizes (\ref{e1}) while satisfying the constraints (\ref{e2}). We will use \(\pi\) and \(\theta\) interchangeably henceforth.

\begin{assumption}
\label{a1}
The policies \(\pi_{\theta} =\{\mu_{\theta_h},h=0,1,\ldots,H-1\}\) are twice continuously differentiable functions of \(\theta=(\theta_0,\theta_1,\ldots,\theta_{H-1})^T\). Moreover, \(\mu_{\theta_h}(s,a) >0\), \(\forall h=0,1,\ldots,H-1, s\in S, a\in A\).
\end{assumption}

\begin{remark}
\label{r1}
Assumption~\ref{a1} is seen to be satisfied for example by the Gibbs distribution, see Proposition 11 of \cite{lakshminarayanan2017}. From the expression of $d^\pi_h(s)$ in \eqref{dpi}, under the setting of parameterized policies, it is easy to see under Assumption ~\ref{a1} that \(d^{\theta}_h(s)\) is continuously differentiable,  $\forall h=0,1,\ldots,H$.
\end{remark}

We have the following preliminary result along the lines of the policy gradient theorem for infinite horizon problems (cf.~Chapter 13 of \cite{sutton2018}).

\begin{theorem}
(The Policy Gradient Theorem for Finite Horizon Constrained MDPs).
\label{policy-gradient-theorem}
Under Assumption~\ref{a1}, for any baseline function \(b:S\rightarrow\mathbb{R}\), we have
\[
\begin{split}
\nabla_\theta L(\pi,\lambda) = &(\sum_{s \in S} d^{\pi}_h(s)\sum_{a \in A} \nabla_{\theta_h} \mu_h(s,a)( Q^{\pi,\lambda}_h(s,a) - b(s)),h=0,1,\dots,H-1)^T.
\end{split}
\]
\end{theorem}

\begin{proof}
Notice first that \(\nabla_\theta L(\pi,\lambda) = (\nabla_{\theta_h} L(\pi,\lambda), h=0,1,\ldots,H-1)^T\). Now observe that
\[
\begin{split}
\nabla_{\theta_h} V^{\pi,\lambda}_0(s_0)=\nabla_{\theta_h} \left(\sum_{a_0 \in A} \mu_0(s_0,a_0) Q^{\pi,\lambda}_0(s_0,a_0)\right)
\end{split}
\]
\[
\begin{split}
=\sum_{a_0 \in A} \big[\nabla_{\theta_h} \mu_0(s_0,a_0) Q^{\pi,\lambda}_0(s_0,a_0)+\mu_0(s_0,a_0) \nabla_{\theta_h} Q^{\pi,\lambda}_0(s_0,a_0) \big]
\end{split}
\]
\[
\begin{split}
=\sum_{a_0 \in A} \left[\nabla_{\theta_h} \mu_0(s_0,a_0) Q^{\pi,\lambda}_0(s_0,a_0)+\mu_0(s_0,a_0) \nabla_{\theta_h} \sum_{s_1 \in S} p_0(s_0,a_0,s_1)\left(c^{\lambda}_0(s_0,a_0,s_1)+V^{\pi,\lambda}_1(s_1)\right) \right]
\end{split}
\]
\[
\begin{split}
=\sum_{a_0 \in A} \big[\nabla_{\theta_h} \mu_0(s_0,a_0) Q^{\pi,\lambda}_0(s_0,a_0)+\mu_0(s_0,a_0) \sum_{s_1 \in S} p_0(s_0,a_0,s_1) \nabla_{\theta_h} V^{\pi,\lambda}_1(s_1) \big].
\end{split}
\]
Unrolling the above expression repeatedly  we get:
\begin{equation}
\begin{split}
\nabla_{\theta_h} V^{\pi,\lambda}_0(s_0)=&\sum_{s \in S} Pr(s_0 \rightarrow s,h,\pi)\sum_{a \in A} \nabla_{\theta_h} \mu_h(s,a)Q^{\pi,\lambda}_h(s,a).
\end{split}
\end{equation}
Now the gradient of the Lagrangian \(L(\pi,\lambda)\) w.r.t \(\theta_h\) can be seen to be
\begin{equation}\label{gradient}
\begin{split}
\nabla_{\theta_h} L(\pi,\lambda)=&\sum_{s \in S} \beta(s) \nabla_{\theta_h} V^{\pi,\lambda}_0(s)= \sum_{s \in S} d^{\pi}_h(s)\sum_{a \in A} \nabla_{\theta_h} \mu_h(s,a)Q^{\pi,\lambda}_h(s,a).
\end{split}
\end{equation}

For any baseline \(b(s),s \in S\), it follows that:
\[
\begin{split}
\sum_{s \in S} d^{\pi}_h(s)\sum_{a \in A} \nabla_{\theta_h} \mu_h(s,a) \left( Q^{\pi,\lambda}_h(s,a) - b(s) \right)=\nabla_{\theta_h} L(\pi,\lambda).
\end{split}
\] 
\end{proof}

We approximate all the value functions for the Lagrangian for \(h=0,1,\dots,H\) as \(V^{\pi,\lambda}_h(s) \approx {v^{\pi,\lambda}_h}^T \phi_h(s)\) and all the value functions for the constraint costs \(W^{\pi,(k)}_h(s) \approx {W^{\pi,(k)}_h}^T \phi_h(s)\) where  \(v^{\pi,\lambda}_h:=(v^{\pi,\lambda}_h(l),l=1,...,x_h)^T\) and \(w^{\pi,(k)}_h:=(w^{\pi,(k)}_h(l),l=1,...,x_h)^T\) are the weight vectors and \(\phi_h(s) \in \mathbb{R}^{x_h}\) is the feature vector \(\phi_h(s):=(\phi_h(s)(1),\dots,\phi_h(s)(x_h))^T\) associated with state \(s\) and time instant \(h\). Note that in a given time instant $h$, it is not possible to reach every state in $S$. Hence we define the subset \[S^{\pi}_h=\{s \in S|d^{\pi}_h(s)>0\},\] in other words $S^{\pi}_h$ is the set of states that can be visited at time instant $h$ under policy $\pi$. Since from Assumption~\ref{a1}, $\mu_h(s,a)>0,  \forall (s,a)$, $S^{\pi}_h$ will be the same under any policy $\pi$. Hence we will refer to it as $S_h$. Let \(\Phi_h \in \mathbb{R}^{|S_h| \times x_h}\) denote the feature matrix with $k$th column being \(\phi_h(.)(k)=(\phi_h(s)(k), s\in S_h)^T,h=0,1,\dots,H\). 

\begin{assumption}
\label{a2}
The basis functions \(\{\phi_h(.)(k),k=1,...,x_h\}\) are linearly independent for \(h=0,\dots,H\). Further \(x_h \leq |S_h|\), and $\Phi_h$ has full rank.
\end{assumption}
\begin{remark}
    It appears desirable to have a higher dimension size for the feature matrix. However, one must be careful about Assumption~\ref{a2}, as it is crucial for the stability of the critic recursion (see Section~\ref{convergence}). Note that Assumption~\ref{a2} can be satisfied by slowly increasing the dimension size of both the weight vector and the features as more states are visited over episodes for a time instant $h$ instead of having a large dimension right from the beginning. This is unlike the setting of infinite horizon problems, where the dimension size remains a constant (see \cite{tsitsiklis1997}). 
    
    Note that in infinite horizon problems, under an ergodic policy, the Markov chain settles into a stationary distribution. This is however not the case with finite horizon problems where, in general, only a certain subset of the states will ever be visited depending on the starting state because of the finite nature of the horizon length and the Markov chain does not enter into a stationary distribution. Hence, the probability of visit to many of the states starting from a given state is actually zero in the finite horizon setting. 
\end{remark}
\section{Actor-Critic Algorithm for Finite Horizon Constrained MDPs}
\label{actor-critic}

We present an algorithm based on multi-timescale stochastic approximation (MTSA) for our problem. Let \(a(n),b(n),c(n),n \geq 0\), be three positive step-size schedules that satisfy the following assumption:
\begin{assumption}
\label{a4}
The step-size sequences \(\{a(n)\}, \{b(n)\}, \{c(n)\}\) satisfy the following properties:
\[\sum_n a(n)=\sum_n b(n) = \sum_n c(n) = \infty,\] 
\[\sum_n (a(n)^2+b(n)^2+c(n)^2) <\infty,\] \[\lim_{n \rightarrow \infty} \frac{b(n)}{a(n)} = \lim_{n \rightarrow \infty} \frac{c(n)}{b(n)} = 0.\]

\end{assumption} 

Our algorithm performs updates once after each episode of length \(H\). Note that in our setting, \(H\) is finite and deterministic. For each episode \(n \geq 0\), let \(\theta_h(n), h=0,\dots,H-1\) denote the running updates of the policy parameter \(\theta_h\) and let \(\lambda^{(k)}(n),k=1,...,M\) denote the running updates of the Lagrange multiplier \(\lambda^{(k)}\). Let \(\theta(n)\stackrel{\triangle}{=}(\theta_0(n),\dots,\theta_{H-1}(n))^T\) and \(\lambda(n)=(\lambda^{(1)}(n),\dots,\lambda^{(M)}(n))^T\). At every instant \(h=0,1,\dots,H-1\), action \(a_h(n)\) is sampled from the distribution \(\mu_{\theta_h(n)}(s_h(n),\cdot)\) when the state is \(s_h(n)\). The episode ends in state \(s_H(n)\). 

\begin{algorithm}[tb]\label{algorithm1}
\caption{Constraint Estimation using critic}
\textbf{Input}: step-sizes
\begin{algorithmic}[1]
\FOR{each episode \(n \geq 0 \)}
\STATE Collect the states, actions and rewards from the episode.
\FOR{\(h=0,1,\dots,H-1\)}
\STATE Perform main critic update (\ref{algo1})
\STATE Perform main actor update (\ref{algo3})
\ENDFOR
\STATE Perform main critic update (\ref{algo2})
\FOR{\(k=1,\dots,M\)}
\FOR{\(h=0,1,\dots,H-1\)}
\STATE Perform constraint critic update (\ref{algo4})
\ENDFOR
\STATE Perform constraint critic update (\ref{algo5})
\STATE Perform Lagrange multiplier update (\ref{algo7})
\ENDFOR
\ENDFOR
\end{algorithmic}
\end{algorithm}

We define temporal difference error \(\delta_h(n),h=0,\dots,H-1\) using the Lagrangian as the single-stage reward. The value function weights \(v^{\pi,\lambda}_h\) estimated using \(v_h(n)\), \(h=0,1,\dots,H\) are updated as in \eqref{algo1}-\eqref{algo2} while the policy parameter is updated according to \eqref{algo3}.
\[
\begin{split}
\delta_h(n)=&c^{\lambda(n)}_h(s_h(n),a_h(n),s_{h+1}(n))+v_{h+1}(n)^T\phi_{h+1}(s_{h+1}(n))
-v_h(n)^T\phi_h(s_h(n)),
\end{split}
\]
\begin{equation}\label{algo1}
v_h(n+1)=v_h(n)+a(n)\delta_h(n)\phi_h(s_h(n)),
\end{equation}
\begin{equation}\label{algo2}
\begin{split}
v_H(n+1)=&v_H(n)+a(n)\big[c^{\lambda(n)}_H(s_H(n))-v_H(n)^T\phi_H(s_H(n))\big]\phi_H(s_H(n)).
\end{split}
\end{equation}

In most settings involving parameterized policies, the parameter naturally takes values in a prescribed compact set. Hence, 
we define a projection operator \(\Gamma\), which projects the iterates to a given compact set. The projection has the additional advantage that it ensures  stability of the policy iterates.  Let \(\psi_{\theta_h}(s,a):=\nabla_{\theta_h}\log \mu_{\theta_h}(s,a),s \in S, a \in A\) denote the compatible features (cf.~\cite{sutton1999policy}). The actor update for \(h=0,1,\dots,H-1\) is then  given by
\begin{equation}\label{algo3}
\begin{split}
\theta_h(n+1)=\Gamma \big[\theta_h(n) + b(n) \psi_{\theta_h(n)}(s_h(n),a_h(n))\delta_h(n)\big].
\end{split}
\end{equation}

The updates are justified by the approximate policy gradient equation \eqref{approx-gradient} and the statements that follow. We define the temporal difference error \(\xi^{(k)}_h(n),h=0,\dots,H-1\) using single stage constraint cost for the $k$th constraint \(k=1,\dots,M\). The value function weights \(w^{\pi,(k)}_h\) estimated using \(w^{(k)}_h(n)\), \(h=0,1,\dots,H\) are updated as in \eqref{algo4}-\eqref{algo5}.
\[
\begin{split}
\xi^{(k)}_h(n)=&g^{(k)}_h(s_h(n),a_h(n),s_{h+1}(n))+w^{(k)}_{h+1}(n)^T\phi_{h+1}(s_{h+1}(n))-w^{(k)}_h(n)^T\phi_h(s_h(n)),
\end{split}
\]
\begin{equation}\label{algo4}
w^{(k)}_h(n+1)=w^{(k)}_h(n)+a(n)\xi^{(k)}_h(n)\phi_h(s_h(n)),
\end{equation}
\begin{equation}\label{algo5}
\begin{split}
w^{(k)}_H(n+1)=&w^{(k)}_H(n)+a(n)\big[g^{(k)}_H(s_H(n))-\alpha^{(k)}-w^{(k)}_H(n)^T\phi_H(s_H(n))\big]\phi_H(s_H(n)).
\end{split}
\end{equation}

We define a projection operator \((\cdot)^-:\mathbb{R} \rightarrow [P,0]\) as \((x)^-=\min(0,\max(x,P))\), where \(x \in \mathbb{R}\) and \(-\infty < P\) is a large negative constant. We provide the Lagrange parameter update below.

\begin{equation}\label{algo7}
\begin{split}
\lambda^{(k)}(n+1)=\big(\lambda^{(k)}(n) - c(n)w^{(k)}_0(n)^T\phi_0(s_0(n))\big)^-.
\end{split}
\end{equation}
The update steps are summarized in Algorithm 1.
\section{Convergence Analysis}
\label{convergence}

Both recursions \eqref{algo3} and \eqref{algo7} clearly satisfy the stability requirement, i.e., $\sup_n \|\theta_h(n)\| <\infty,h=0,\dots,H-1$ and $\sup_n \|\lambda^{(k)}(n)\| <\infty,k=1,\dots,M$ almost surely, since the projection operators $\Gamma(\cdot)$ and $(\cdot)^-$ force the iterates to evolve within a compact set. Further, using arguments as in Chapter 6 of \cite{borkar2008}, one may let $\lambda(n) \equiv \lambda$ when analyzing the actor update \eqref{algo3} and $\theta(n) \equiv \theta$ when analyzing the critic update \eqref{algo1}-\eqref{algo2} and \eqref{algo4}-\eqref{algo5} since $c(n)=o(a(n))$ and moreover $b(n)=o(a(n))$ from Assumption~\ref{a4}.

Let \(P^{\theta}_h\) be the probability transition matrix for \(h=0,1,\dots,H-1\) with elements \(P^{\theta}_h(s,s')=\sum_{a \in A} \mu_{\theta_h}(s,a)p_h(s,a,s'),s \in S_h, s' \in S_{h+1}\). Let \(D^{\theta}_h\) denote a diagonal matrix for \(h=0,1,\dots,H\) with elements \(d^{\theta}_h(s), s \in S_h\) along the diagonal. This is unlike the infinite horizon case, where the stationary distribution of all the states in $S$ are taken (see \cite{tsitsiklis1997}).
Let us define the vectors \(C^{\theta,\lambda}_H\), \(C^{\theta,\lambda}_h\), \(G^{\theta,(k)}_H\) and \(G^{\theta,(k)}_h\), \(h=0,\dots,H-1\), \(k=1,\dots,M\), as
\[
C^{\theta,\lambda}_H=\left(c^{\lambda}_H(s), s \in S\right)^T,
\]
\[
C^{\theta,\lambda}_h=\left(\sum_{a \in A} \mu_{\theta_h}(s,a) \sum_{s' \in S} p_h(s,a,s') c^{\lambda}_h(s,a,s'), s \in S\right)^T.
\]
\[
G^{\theta,(k)}_H=\left(g^{(k)}_H(s)-\alpha^{(k)}, s \in S\right)^T,
\]
\[
G^{\theta,(k)}_h=\left(\sum_{a \in A} \mu_{\theta_h}(s,a) \sum_{s' \in S} p_h(s,a,s') g^{(k)}_h(s,a,s'), s \in S\right)^T.
\]
Let us define the following points, which is used in Theorem~\ref{th2}:
\[
\Lambda_H(\theta)=(\Phi_H^T D_H^{\theta}\Phi)^{-1}\Phi_H^T D_H^{\theta}C^{\theta,\lambda}_H,
\]
\[
\Lambda_h(\theta)=(\Phi_h^T D_h^{\theta}\Phi_h)^{-1} (\Phi_h^T D_{h}^{\theta} P_h^{\theta} \Phi_{h+1} \Lambda_{h+1} + \Phi_h^T D_h^{\theta} C^{\theta,\lambda}_h),
\]
\[
\Xi^{(k)}_H(\theta)=(\Phi_H^T D_H^{\theta}\Phi_H)^{-1}\Phi_H^T D_H^{\theta}G^{\theta,(k)}_H,
\]
\[
\Xi^{(k)}_h(\theta)=(\Phi_h^T D_h^{\theta}\Phi_h)^{-1} (\Phi_h^T D_{h}^{\theta} P_h^{\theta} \Phi_{h+1} \Xi^{(k)}_{h+1} + \Phi_h^T D_h^{\theta} G^{\theta,(k)}_h).
\]

\begin{theorem}\label{th2}
For \(\lambda(n) \equiv \lambda\) and \(\theta(n) \equiv \theta\), $v_h(n),h=0,\dots,H-1$ and $v_H(n)$ converge to points \(\Lambda_h(\theta),h=0,\dots,H-1\) and \(\Lambda_H(\theta)\) respectively. Similarly $w^{(k)}_h(n),h=0,\dots,H-1$ and $w^{(k)}_H(n)$ converge to points \(\Xi^{(k)}_h(\theta),h=0,\dots,H-1\) and \(\Xi^{(k)}_H(\theta)\), respectively, for \(k=1,\dots,M\) almost surely. Also the points \(\Lambda_h(\theta),h=0,\dots,H-1\), \(\Lambda_H(\theta)\), \(\Xi^{(k)}_h(\theta),h=0,\dots,H-1\) and \(\Xi^{(k)}_H(\theta)\) for \(k=1,\dots,M\) are Lipschitz continuous in $\theta$ and $\lambda$.
\end{theorem}
\begin{proof}
We prove the convergence of \(v_h(n),h=0,\dots,H-1\) and \(v_H(n)\). The convergence of \(w^{(k)}_h(n),h=0,\dots,H-1\) and \(w^{(k)}_H(n)\) for \(k=1,\dots,M\) follows similarly. The recursions \eqref{algo1}-\eqref{algo2} can be written as,
\begin{equation}\label{rec1}
\begin{split}
v_h(n+1)=v_h(n) + a(n)\big[f_h(v_h(n),v_{h+1}(n))+N_h^1(n+1)\big],
\end{split}
\end{equation}
\begin{equation}\label{rec2}
\begin{split}
v_H(n+1)=v_H(n) + a(n)\big[f_H(v_H(n))+N_H^1(n+1)\big]
\end{split}
\end{equation}
$h=0,1,\ldots,H-1$. In the above, 
\[
\begin{split}
f_h(v_h,v_{h+1})=\sum_{s \in S} d^{\theta}_h(s)\sum_{a \in A} \mu_{\theta_h}(s,a) \sum_{s' \in S} p_h(s,a,s')\big[ c^{\lambda}_h(s,a,s') + v_{h+1}^T \phi_{h+1}(s') - v_h^T \phi_h(s) \big] \phi_h(s), 
\end{split}
\]
\[
f_H(v_H)=\sum_{s \in S} d^{\theta}_H(s)\big[c^{\lambda}_H(s)-v_H^T \phi_H(s)\big] \phi_H(s), 
\]
The noise terms can be derived from above. \(f_h\) is linear in \(v_h\) and \(v_{h+1}\) and \(f_H\) is linear in \(v_H\), therefore they are Lipschitz continuous. Let \(\mathcal{F}(n), n\geq 0\) be the increasing sequence of sigma fields
\[
\begin{split}
\mathcal{F}(n)=\sigma(v_h(m),v_H(m),N^1_h(m),N^1_H(m),m \leq n,h=0,1,\dots,H-1).
\end{split}
\]
For some \(0 < C<\infty\), it is easy to see from Assumption~\ref{a2},
\[
\begin{split}
\mathbb{E}\left[||N_h^1(n+1)||^2|\mathcal{F}(n)\right]\leq& C(1+||v_h(n)||^2+||v_{h+1}(n)||^2),h=0,1,\dots,H-1,
\end{split}
\]
\[
\mathbb{E}\left[||N_H^1(n+1)||^2|\mathcal{F}(n)\right]\leq C(1+||v_H(n)||^2).
\]
Let us define:
\begin{equation}\label{limit1}
\begin{split}
f_h^{\infty}(v_h,v_{h+1})&=\lim_{c \rightarrow \infty}\frac{f_h(cv_h,cv_{h+1})}{c}=\Phi_h^T D^{\theta}_h P^{\theta}_h \Phi_{h+1} v_{h+1}-\Phi_h^T D^{\theta}_h \Phi_h v_h,
\end{split}
\end{equation}
\begin{equation}\label{limit2}
f_H^{\infty}(v_H)=\lim_{c \rightarrow \infty} \frac{f_H(cv_H)}{c}=-\Phi_H^T D^{\theta}_H \Phi_H v_H.
\end{equation}
Clearly \(D^{\theta}_h\) is positive definite. Since from Assumption~\ref{a2}, \(\Phi_h\) has full rank then \(-\Phi_h^T D^{\infty}_h \Phi_h\) is negative definite for \(h=0,\dots,H\). Hence the ODE \(\dot{v}_H(t)=f_H^{\infty}(v_H(t))\) has the origin in \(\mathbb{R}^{x_H}\) as its unique asymptotically stable equilibrium. Now plugging $v_{h+1}(t)=0$, one can see going backwards \(h=H-1,\dots,0\), each ODE \(\dot{v}_h(t)=f_h^{\infty}(v_h(t),v_{h+1}(t))\) has the origin in \(\mathbb{R}^{x_h}\) as its unique asymptotically stable equilibrium. Consider the ODEs:
\begin{equation}
\label{finite_ode}
\begin{split}
&\dot{v}_h(t)=f_h(v_h(t),v_{h+1}(t)),h=0,\dots,H-1\\
&\dot{v}_H(t)=f_H(v_H(t))
\end{split}
\end{equation}
We show that $\Lambda_h,h=0,\dots,H$ is the unique globally asymptotically stable equilibrium of \eqref{finite_ode} with 
\[
\begin{split}
&\mathcal{L}_h(v_h,v_{h+1})=\frac{1}{2}f_h(v_h,v_{h+1})^T f_h(v_h,v_{h+1}),h=0,\dots,H-1\\
&\mathcal{L}_H(v_H)=\frac{1}{2}f_H(v_H)^T f_H(v_H)
\end{split}
\]
serving as an associated strict lyapunov function. Thus
\[
\begin{split}
\frac{d\mathcal{L}_h(v_h(t),v_{h+1}(t))}{dt}=\nabla_{v_h} \mathcal{L}_h(v_h,v_{h+1})\dot{v}_h=f_h(v_h,v_{h+1})^T(-\Phi_h^T D^{\theta}_h \Phi_h )f_h(v_h,v_{h+1}) \leq 0,\\
h=0,\dots,H-1\\
\frac{d\mathcal{L}_H(v_H(t))}{dt}=\nabla_{v_H} \mathcal{L}_H(v_H)\dot{v}_H=f_H(v_H)^T(-\Phi_H^T D^{\theta}_H \Phi_H )f_H(v_H)\leq 0.
\end{split}
\]
Hence (A1) and (A2) of \cite{borkar2000} are satisfied and the first claim follows from Theorem 2.2 of \cite{borkar2000}. The iterates $v_h(n)$ converges where the approximate TD error $f_h$ is 0, for $h=0,\dots,H$. 

From Assumption~\ref{a1}, $D^{\theta}_h$ and $C^{\theta,\lambda}_h$ are continuously differentiable in $\theta$ and $C^{\theta,\lambda}_h$ is linear in $\lambda$ for $h=0,\dots,H$. Since $\theta$ and $\lambda$ lie in a compact space, $d^{\theta}_h(s) \geq B_1,B_1>0$, therefore $\nabla_{\theta}(\Phi_h^T D_h^{\theta}\Phi_h)^{-1}$, $\nabla_{\theta}D^{\theta}_h$ and $\nabla_{\theta}C^{\theta,\lambda}_h$ are bounded. Therefore, $\nabla_{\theta}\Lambda_h(\theta)$ is bounded. Similarly one can show $\nabla_{\theta}\Xi_h^{(k)}(\theta)$ is bounded for $h=0,\dots,H$. The second claim follows.
\end{proof}

Let the approximate policy gradient be defined as,
\begin{equation}\label{approx-gradient}
\begin{split}
g_h(\theta)=\sum_{s \in S} d^{\theta}_h(s)\sum_{a \in A} \mu_{\theta_h}(s,a)\nabla_{\theta_h} \log{\mu_{\theta_h}(s,a)}\sum_{s' \in S} p_h(s,a,s')\\
\big(c^{\lambda}_h(s,a,s') + \Lambda_{h+1}(\theta)^T \phi_{h+1}(s') - \Lambda_h(\theta)^T \phi_h(s) \big).
\end{split}
\end{equation}
This is the gradient given in \eqref{gradient}, with baseline $b(s):=V^{\pi,\lambda}_h(s),s \in S$, baring that approximate value functions \(V^{\pi,\lambda}_h(s) \approx {\Lambda_h(\theta)}^T \phi_h(s)\) are used here. The error in gradient due to approximation is defined as:
\[
\mathcal{E}_h^{\theta}=g_h(\theta)-\nabla_{\theta_h}L(\theta,\lambda)
\]
The directional derivative of \(\Gamma(\cdot)\) at point $x$, along the vector $y$ is defined as,
\begin{equation}\label{dd}
\Gamma(x,y)'=\lim_{\eta \downarrow 0} \frac{\Gamma(x+\eta y)-x}{\eta}
\end{equation}
Let us define the set:
\(\kappa(\lambda)=\{(\theta_h \in \mathbb{R}^{y_h},h=0,\dots,H-1)^T|\Gamma(\theta_h,\nabla_{\theta_h}L(\theta,\lambda))'=0\forall h=0,\dots,H-1\}\). This is the set of $\theta$ such that the policy gradient in \eqref{gradient} is 0 $\forall h=0,\dots,H-1$ or at a boundary of the projection set of $\Gamma$, which means $\theta$ is at a bounded local optimum. Let $\kappa^{\epsilon}(\lambda)$ be the closed $\epsilon$-neighborhood of $\kappa(\lambda)$.

\begin{theorem}\label{th3}
For $\lambda(n) \equiv \lambda$, given $\epsilon>0$, $\exists \delta>0$ such that if $\sup_{\theta}||\mathcal{E}_h^{\theta}||\leq\delta$ $\forall h=0,\dots,H-1$ then $\theta(n)$ converges to the set $\kappa^{\epsilon}(\lambda)$ almost surely.
\end{theorem}
\begin{proof}
The recursion \eqref{algo3} can be written as,
\begin{equation}\label{rec3}
\begin{split}
\theta_h(n+1)=\Gamma \big(\theta_h(n) + b(n)\big(g_h(\theta(n))+N_h^2(n+1)+e^{\theta(n)}\big)\big),
\end{split}
\end{equation}
where,
\[
\begin{split}
N^2_h(n+1)=\nabla_{\theta_h} \log{\mu_{\theta_h}(s,a)}\big(c^{\lambda}_h(s,a,s') + \Lambda_{h+1}(\theta(n))^T \phi_{h+1}(s')- \Lambda_h(\theta(n))^T \phi_h(s) \big)-g_h(\theta(n)).
\end{split}
\]
The error $e^{\theta(n)}$ can be derived from above and $e^{\theta(n)}\rightarrow 0$ as $n \rightarrow 0$ (from Theorem~\ref{th2}). It is easy to see \(N^2_h(n+1)\) is uniformly bounded, since $\Lambda_h(\theta)$ and $\Lambda_{h+1}(\theta)$ is uniformly bounded from Theorem~\ref{th2}. Therefore the martingale \(M_1(n)=\sum_{r=0}^{n-1}b(r)N^2_h(r+1)\) converges almost surely (see Appendix C of \cite{borkar2008}). Therefore the tail sum \(M_1'=\sum_{r=n}^{\infty}b(r)N^2_h(r+1) \rightarrow 0\) as \(n \rightarrow \infty\). Now from Theorem~\ref{th2}, $\Lambda_h(\theta)$ and $\Lambda_{h+1}(\theta)$ are both Lipschitz continuous in $\theta$. Since $\theta$ lies in a compact space, from Assumption~\ref{a1}, \(\nabla^2_{\theta} \mu_{\theta_h}(s,a)\) is bounded. Hence $g_h(\theta)$ is lipschitz continuous. Consider the associated ODEs for $h=0,\dots,H-1$:
\begin{equation}\label{finite_actor_ode1}
\dot{\theta}_h(t)=\Gamma(\theta_h(t),g_h(\theta(t)))'
\end{equation}
\begin{equation}\label{finite_actor_ode2}
\dot{\theta}_h(t)=\Gamma(\theta_h(t),\nabla_{\theta_h}L(\theta(t),\lambda))'
\end{equation}
Therefore by Theorem 5.3.1 of \cite{kushner1978}, \(\theta_h(n)\) asymptotically tracks the ODE \eqref{finite_actor_ode1}. Clearly $\nabla^2_{\theta} L(\theta,\lambda)$ is bounded hence $\nabla_{\theta_h}L(\theta,\lambda)$ is Lipschitz continuous. If $\sup_\theta||\mathcal{E}_h^{\theta}||\leq\delta$ $\forall h=0,\dots,H-1$ then, $\theta_h(n)$ is a perturbed trajectory of \eqref{finite_actor_ode2}. Now $\kappa(\lambda)$ is the set of stable equilibria of \eqref{finite_actor_ode2}. The claim follows from Theorem 1 of \cite{hirsch1989}.
\end{proof}

Let \((\cdot)^{-'}\) be the directional derivative of \((\cdot)^-\) defined as \eqref{dd}. Let us define the set:
\(F=\{\lambda \in [P,0]^M|\big(\lambda,\sum_{s_0 \in S} \beta(s_0){\Xi_0^{(k)}(\theta^\lambda)}^T \phi_0(s_0)\big)^{-'}=0, \forall k=1,\dots,M,\theta^\lambda \in \kappa^\epsilon(\lambda)\}\). This set contains all the lagrange parameter $\lambda$ for which the approximate constraint violation is 0. Note that the point $\Xi_0^{(k)}$ implicitly depends on $\lambda$.

\begin{theorem}\label{th4}
As $n\rightarrow\infty$, \(\lambda(n)=(\lambda^{(1)}(n),\ldots,\lambda^{(M)}(n))^T\) obtained from (\ref{algo7}) converges to a set $F$ almost surely.
\end{theorem}

\begin{proof}
The Lagrange parameter update (\ref{algo7}) can be rewritten as,
\begin{equation}\label{lrec}
\begin{split}
\lambda^{(k)}(n+1)=\big(\lambda^{(k)}(n)-c(n)\big({\Xi_0^{(k)}(\theta^\lambda)}^T \phi_0(s_0(n))+\gamma^{(k)}(n)\big)\big)^-
\end{split}
\end{equation}
where \(\gamma^{(k)}(n)\) can be derived from above and \(\gamma^{(k)}(n) \rightarrow 0\) as \(n \rightarrow \infty\) (from Theorem~\ref{th2}). Thus (\ref{lrec}) is an Euler discretization with (non-uniform) step-sizes $c(n)$ of the ODE:
\begin{equation}\label{finite_lag_ode}
\dot{\lambda}^{(k)}(t)=\big(-\sum_{s_0 \in S} \beta(s_0){\Xi_0^{(k)}(\theta^{\lambda(t)})}^T \phi_0(s_0)\big)^{-'}.
\end{equation}
Now $F$ is the set of stable equilibria of \eqref{finite_lag_ode} using the lyapunov function $\max_\theta\sum_{s_0 \in S} \beta(s_0){\Lambda_0}(\theta)^T \phi_0(s_0)$. The claim follows from Theorem 5.3.1 of \cite{kushner1978}.
\end{proof}

Hence by Theorem~\ref{th2}-\ref{th4}, the algorithm converges asymptotically.
\section{Experiments}
\label{numerical}
\begin{figure*}
\centering
\includegraphics[width=1\textwidth,height=6cm]{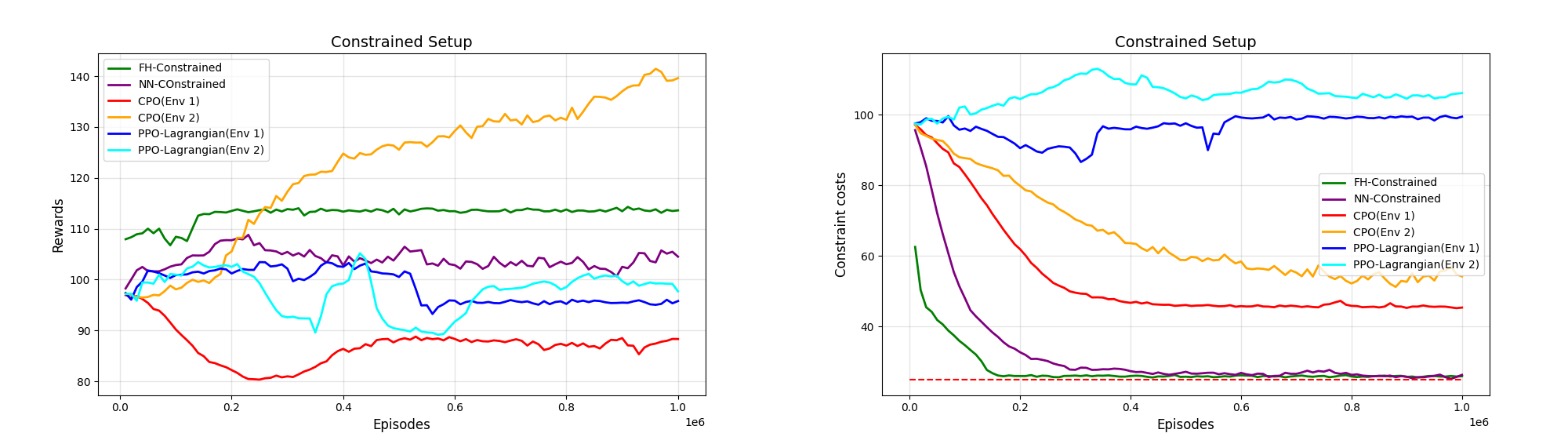}
\caption{Comparison of rewards and constraint costs achieved by different algorithms.}
\label{fig2}
\end{figure*}
Now that we have proved the convergence of our algorithm, we study the empirical performance of our algorithm on a randomly generated 2-dimensional Grid World Problem. We also introduce a deep learning version of our algorithm, where the value and policy function are parameterized using a neural network architecture. Along each dimension, at any instant, the agent can go one step forward or one step back, or else stay in the same place. So for a 2-dimensional grid, there are available $3 \times 3=9$ actions in each state. As far as state transitions are concerned, there is a \(90 \%\) chance that the agent will go one step in the direction suggested by the action and a \(10\%\) chance that it will randomly move one step in any of the other directions. The agent collects a reward if it goes to a particular state. We have a single constraint. There are bad states, which the agent must avoid and there is a constraint cost associated with those states. The position of the rewards and constraints change with time, so the agent needs to make decisions at each time instant accordingly. This kind of setup can be applied in applications such as taxi services, where a driver needs to maximize his/her fare in a day comprising of say \(10\) hours and the grid is of the city in which he/she is driving. The rewards would then correspond to the fares and the constraints could be on the level of traffic congestion encountered on the different routes. The congestion levels are clearly dynamic in nature as they would naturally change with time.

 We let a horizon length of \(H=100\) for our experiments. Because of our constrained setup, our objective is to maximize the rewards while satisfying the inequality constraint. We name our finite horizon constrained algorithms as FH-Constrained. We also name our corresponding deep learning algorithms as NN-Constrained. In NN-Constrained, we use fully connected feed-forward neural network for both critic and actor. Each network has two hidden layers, 10 neurons per layer, tanh activation function, and 2-dimensional state space as input. 
 
 For NN-Constrained, we generate batches of $K$ trajectories using the current policy $\pi(n)=\{\mu_0(n),\dots,\\\mu_{H-1}(n)\}$. The data for $i$-th trajectory of $n$-th batch is of the form:$s_0^i(n),a_0^i(n),c_0^i(n),g_0^i(n),s_1^i(n),\dots$ which are state, action, lagrangian reward (equivalent to $c^\lambda_h$), constraint cost, next state and so on. The critic loss(CL), constraint critic loss(CCL) and actor loss(AL) for $n$-th batch of data are the following:
 \[
 \begin{split}
 &CL(v_h)=\frac{1}{K}\sum_{i=1}^K(c_h(n)+V_{h+1}(s_{h+1}^i(n))|_{v_{h+1}=v_{h+1}(n)}-V_h(s_h^i(n)))^2,h=0,\dots,H-1\\
 &CL(v_H)=\frac{1}{K}\sum_{i=1}^K(c_H(n)-V_H(s_H^i(n)))^2\\
 &CCL(w_h)=\frac{1}{K}\sum_{i=1}^K(g_h(n)+W_{h+1}(s_{h+1}^i(n))|_{w_{h+1}=w_{h+1}(n)}-W_h(s_h^i(n)))^2,h=0,\dots,H-1\\
 &CCL(w_H)=\frac{1}{K}\sum_{i=1}^K(g_H(n)-W_H(s_H^i(n)))^2
 \end{split}
 \]
 \[
 \begin{split}
 AL(\theta_h)=\frac{1}{K}\sum_{i=1}^K\log{\mu_h(s^i_h(n),a^i_h(n))}(c_h(n)+V_{h+1}(s_{h+1}^i(n))|_{v_{h+1}=v_{h+1}(n)}-V_h(s_h^i(n))|_{v_h=v_h(n)}),\\
 h=0,\dots,H-1
 \end{split}
 \]
 Other implementation details can be found in the code.
 We compare our algorithm with Constrained Policy Optimization (CPO) \cite{achiam2017} and PPO-Lagrangian \cite{ray2019}, which are the best known algorithms for Constrained RL. Each comparison is done on two environments, one with $2$-dimensional state space as input and another with $3$-dimensional input (two dimensions for the state variable and one dimension for the time instant). 
 The first plot in Fig.1 is for the aggregate reward while the second is for the constraint cost. We take the total reward/constraint cost of each episode and average the same over the last 10,000 episodes. Each setting is run 5 times with independent seeds. The red dotted line in the second plot denotes the constraint threshold of \(\alpha=25\).

  We can see that our algorithm performs better than both CPO and PPO-Lagrangian as it nearly satisfies the constraint every time while others do not. Also note that CPO (Env 2) gets more rewards than our algorithms but that comes at a cost of it violating the constraints by a wide margin. It is clear from the experiments that constrained algorithms for infinite horizon problems are not appropriate for finite horizon settings. Also feeding the time instant with state as input does not  help in these algorithms. In fact, they perform poorly when the reward and the constraint cost change with time, unlike our algorithm that easily adapts to such changes.

\section{Conclusions}
\label{conclusions}

We presented the first policy gradient reinforcement learning algorithm for Finite Horizon Constrained MDPs. One must adhere to our algorithm when the agent needs to take time critical decisions, which is evident from our empirical comparisons with other well known algorithms. Our algorithm involves three timescale schedules and is of the actor-critic type. We provided the full asymptotic convergence analysis for our algorithm. The power of our algorithm comes from using separate parametric functions for each time instant, and our convergence analysis shows that they can successfully interact with each other to learn a time critical task. It will be of interest to analyze the sample complexity of our constrained three-timescale finite horizon MDP algorithm. Further, it will be interesting to perform a finite sample analysis of Actor critic algorithms for Finite Horzion CMDPs as future work. One may also come up with sophisticated policy optimization methods such as CPO \cite{achiam2017}, for finite horizon MDPs.
\bibliographystyle{IEEEtran}
\bibliography{IEEEfull}

\end{document}